%% file: main.tex
\documentclass{article}

\pdfoutput=1

\usepackage{microtype}
\usepackage{graphicx}
\usepackage{caption}
\usepackage{subcaption}
\usepackage{booktabs} 
\usepackage{tikz}
\usepackage{tablefootnote}
\usepackage{pgfplots, pgfplotstable}
\usepgfplotslibrary{groupplots}
\usepgfplotslibrary{fillbetween}
\usepgfplotslibrary{statistics}
\usepackage{amsmath}

\usepackage{hyperref}

\usepackage{hos}

\usepackage[accepted]{icml2019}

\icmltitlerunning{Parsimonious Black-Box Adversarial Attacks via Efficient Combinatorial Optimization}

\begin{document}

\twocolumn[
\icmltitle{Parsimonious Black-Box Adversarial Attacks\\ via Efficient Combinatorial Optimization}



\icmlsetsymbol{equal}{*}

\begin{icmlauthorlist}
\icmlauthor{Seungyong Moon}{equal,snu,nprc}
\icmlauthor{Gaon An}{equal,snu,nprc}
\icmlauthor{Hyun Oh Song}{snu,nprc}
\end{icmlauthorlist}

\icmlaffiliation{snu}{Department of Computer Science and Engineering, Seoul National University, Seoul, Korea}
\icmlaffiliation{nprc}{Neural Processing Research Center}

\icmlcorrespondingauthor{Hyun Oh Song}{hyunoh@snu.ac.kr}

\icmlkeywords{Adversarial attack, Submodularity}

\vskip 0.3in
]



\printAffiliationsAndNotice{\icmlEqualContribution} 

\begin{abstract}

Solving for adversarial examples with projected gradient descent has been demonstrated to be highly effective in fooling the neural network based classifiers. However, in the black-box setting, the attacker is limited only to the query access to the network and solving for a successful adversarial example becomes much more difficult. To this end, recent methods aim at estimating the true gradient signal based on the input queries but at the cost of excessive queries. We propose an efficient discrete surrogate to the optimization problem which does not require estimating the gradient and consequently becomes free of the first order update hyperparameters to tune. Our experiments on Cifar-10 and ImageNet show the state of the art black-box attack performance with significant reduction in the required queries compared to a number of recently proposed methods. The source code is available at \url{https://github.com/snu-mllab/parsimonious-blackbox-attack}.
\end{abstract}

\section{Introduction}\label{sec:intro}

Understanding the vulnerability of neural network based classifiers to adversarial perturbations \cite{szegedy13, carlini16} designed to fool the classifier predictions have emerged as an important research area in machine learning. Recent studies have devised highly successful attacks in the \emph{white-box} setting \cite{fgsm,pgd,carlini17}, where the attacker has access to the network model parameters and the corresponding loss gradient with respect to the input perturbation. 

However, in more realistic settings \cite{watson,googleapi,clarifai}, the attacker is limited to the access of input queries and the corresponding model predictions. These web services offer various commercial recognition services such as image classification, content moderation, and face recognition backed up by pretrained neural network based classifiers. In this setting, the inference network receives the query image from the user and only exposes the inference results, so the model weights are hidden from the user.

To this end, \emph{black-box} methods construct adversarial perturbations without utilizing the model parameters or the gradient information. Some recent works on black-box attacks compute the gradient using a substitute network \cite{papernot16,papernot17} but it has been shown that the method does not always transfer to the target network. On the other hand, another line of works aim at estimating the gradient based on the model predictions from the input queries and apply first order updates with the estimated gradient vector \cite{zoo,autozoom,bhagoji18,nes,bandit}. However, the robustness of this approach can be susceptible to the choice of the hyperparameters such as the learning rate, decay rates, and the update rule since the performance of the method hinges on the first order updates with approximated ascent directions. 

We first consider a discrete surrogate problem which finds the solution among the \emph{vertices} of the $\ell_\infty$ ball and show that this unlocks a new class of algorithms which constructs adversarial perturbations without the need to estimate the gradient vector. This comes with the benefit that the algorithm becomes free of the update hyperparameters and thus is more applicable in black-box settings. Intuitively, our proposed method defines and maintains upper bounds on the \emph{marginal gain} of attack locations, and this leads to a parsimonious characteristic of the algorithm to terminate quickly without having to perform excessive queries.

Our results on Cifar-10 \cite{cifar} and ImageNet \cite{imagenet} show the state of the art attack performance under $\ell_\infty$ noise constraint demonstrating significantly higher attack success rates while making considerably less function queries compared to the recent baseline methods \cite{zoo, autozoom, nes, bandit} in both the untargeted and targeted black-box attack settings. Notably, our method achieves attack success rate comparable to the white-box PGD attack \cite{pgd} in some settings (attacks on the adversarially trained network on Cifar-10), although the method uses more queries being a black-box attack method.

\section{Related works}\label{sec:rel}
There has been a line of work on adversarial attacks after the recent discovery of network vulnerability to the attacks from \citet{biggio12,szegedy13}. In our paper, we focus on black-box attacks under $\ell_\infty$ constraint with access to the network prediction scores only. Although attacks on more limited settings (access to the network decision only) have been explored \cite{brendel18,nes,cheng18}, the methods typically require up to $1M$ queries per images which can be difficult to apply in practical settings.

\textbf{Black-box attacks with substitute networks} \citet{papernot16,papernot17} utilize separate substitute networks trained to match the prediction output of the target network similar to model distillation \cite{hinton_distillation}. The idea then is to craft adversarial perturbations by using the backpropagation gradients from the substitute network and transfer the adversarial noise to the target network. The follow-up work from \citet{liu17} showed that black-box attacks with substitute networks tend not to transfer well for targeted attacks but can be improved with ensemble networks. However, the attack success rates for these methods are outperformed by another line of techniques which directly estimate the gradient of the target network based on the input queries.

\textbf{Black-box attacks with gradient estimation} \citet{zoo} computes the coordinate-wise numerical gradient of the target network by repeatedly querying for the central difference values \emph{at each pixels per each ascent steps}. This can be prohibitive as it would require approximately half million queries on moderate sized images. \citet{bhagoji18} mitigates the issue by grouping the pixels at random or via PCA but still requires computing the group-wise numerical gradients per each ascent steps. \\
In contrast, \citet{nes,autozoom} compute the vector-wise gradient estimate with random vector $u_i$ by computing $\frac{1}{\sigma n} \sum_i^n (f(x + \sigma u_i) - f(x - \sigma u_i)) u_i$. \citet{bandit} extends the approach to incorporate time-dependent prior which acts similar to the momentum term \cite{nesterov} in first order optimization and data-dependent prior which exploits the spatial regularity (also in \citet{autozoom}) for query efficiency.

\section{Methods}\label{sec:method}
Suppose we have a classifier with a corresponding loss function $\ell(x,y)$. In black box attacks, the goal is to craft imperceptible adversarial perturbations $(x_{adv})$ typically under small $\ell_\infty$ radius in a limited query budget. Furthermore, the attacker only has access to the loss function (zeroth order oracle). This is a challenging setup as the attacker does not have access to the gradient information (first order oracle) with respect to the input. 

\subsection{Problem formulation}

First order methods tend to have strong attack performance by formulating a constrained optimization problem and querying the gradient of the loss with respect to the input perturbation. Fast gradient sign method (FGSM) \cite{fgsm} first derives the following first order Taylor approximation to the loss function. Note, PGD is a multi-step variant of FGSM \cite{pgd}.
\[\ell(x_{adv},y) \approx \ell(x, y) + (x_{adv} - x)^\intercal \nabla_x \ell(x,y)\]

Then, the optimization problem becomes,
\vspace{-1.5em}

\footnotesize
\begin{align}
\label{eqn:lp}
&\maximize_{\|x_{adv} - x\|_\infty \leq \epsilon}~ \ell(x_{adv}) \implies \maximize_{x_{adv}}~ {x_{adv}}^\intercal \nabla_x \ell(x,y)\\
&\hspace{11.3em}\text{\ subject to}~~ -\epsilon \mathbf{1} \preceq x_{adv} - x \preceq \epsilon \mathbf{1},\nonumber
\end{align}
\normalsize
where $\mathbf{1}$ denotes the vector of ones, $\preceq$ denotes the element-wise inequality. Thus, we can interpret FGSM as finding the solution to the above linear program (LP) in \Cref{eqn:lp} with the gradient vector evaluated at the original image $x$ as the cost vector in LP. Similarly, we can interpret that PGD sequentially finds the solution to the above LP at each step with the updated gradient vector $\nabla_{x_{adv}} \ell(x_{adv}, y)$ as the cost vector.

Since the feasible set in \Cref{eqn:lp} is bounded, the solution of the LP is attained at an extreme point of the feasible set \cite{schrijver86}, and we can theoretically characterize that an optimal solution will be attained at a vertex of the $\ell_\infty$ ball. \Cref{fig:noise_distribution} shows an example statistics of the adversarial noise $(x_{adv} - x)$ obtained by running the PGD algorithm until convergence on Cifar-10 dataset. This shows that the empirical solution from PGD is mostly found on vertices of $\ell_\infty$ ball as well. We also found that running PGD until convergence on ImageNet dataset also produces similar results. 


This characterization together with the fact that in many realistic scenarios, the access to the true gradient is not readily available, motivates us to consider a discrete surrogate to the problem as shown in \Cref{eqn:continuous_form}.
\vspace{-1.5em}

\footnotesize
\begin{align}
\label{eqn:continuous_form}
&\maximize_{x_{adv} \in \reals^{p}}~ f(x_{adv}) ~\qquad\implies~ \maximize_{x_{adv}}~ f(x_{adv})\\
&\text{subject to}~~ \|x_{adv} - x\|_\infty \leq \epsilon ~~~\quad\text{subject to}~~ x_{adv} - x \in \{\epsilon, -\epsilon\}^p\nonumber,
\end{align}
\normalsize
where $p$ denotes the number of pixels in the image $x$, $f(x) = \ell(x, y_{gt})$ for untargeted attacks with ground truth label $y_{gt}$, and $f(x) = -\ell(x, y_{target})$ for targeted attacks with target label $y_{target}$.

Optimizing the discrete surrogate problem does not require estimating the gradient and thus removes all the hyperparameters crucial for gradient update based attacks \cite{nes, bandit, zoo, autozoom}. Furthermore, we show in \Cref{sec:local_search} how the surrogate exploits the underlying problem structure for faster convergence (early-termination in black-box attacks).

Equivalently, the discrete problem in \Cref{eqn:continuous_form} can be reformulated as the following set maximization problem in \Cref{eqn:discrete_form}.
\small
\begin{align}
\label{eqn:discrete_form}
\maximize_{\Scal \subseteq \Vcal}~ \left\{ F(\Scal) \triangleq f\left(x + \epsilon \sum_{i \in \Scal} e_i - \epsilon \sum_{i \notin \Scal} e_i\right)\right\},
\end{align}
\normalsize

where $e_i$ denotes the $i$-th standard basis vector, $\Vcal$ denotes the ground set which is the set of all pixel locations ($|\Vcal| = p$), $\Scal$ denotes the set of \emph{selected} pixels with $+\epsilon$ perturbations, and $\Vcal \setminus \Scal$ indicates the set of remaining pixels with $-\epsilon$ perturbations. The goal in \Cref{eqn:discrete_form} is to find the set of pixels $\Scal$ with $+\epsilon$ perturbations (and vice versa for $\Vcal \setminus \Scal$) which will maximize the objective function. However, finding the exact solution to the problem is NP-Hard as na\"ive exhaustive solution requires combinatorial ($2^{|\Vcal|}$) queries. Subsequent subsections discuss how we exploit the underlying problem structure for efficient computation.

\subsection{Approximate submodularity}

In general, maximizing functions over sets is usually NP-hard \cite{krause12, bachSubmodular}. However, many set functions that arise in machine learning problems often exhibit \emph{submodularity}.

\begin{definition}
For a set function $F: 2^{\Vcal} \rightarrow \reals, \Scal\subseteq \Vcal,$ and $e \in \Vcal$, let $\Delta(e \mid \Scal) := F(\Scal \cup \{e\}) - F(\Scal)$ be the marginal gain \cite{krause12} of $F$ at $\Scal$ with respect to $e$.
\end{definition} 

\begin{definition}
A function $F: 2^{\Vcal} \rightarrow \reals$ is submodular if for every $\Acal \subseteq \Bcal \subseteq \Vcal$ and $e \in \Vcal \setminus \Bcal$ it holds that
\[\Delta(e \mid \Acal) \geq \Delta (e \mid \Bcal)\]
\end{definition}

\begin{figure}[t]
\begin{minipage}[t]{0.46\columnwidth}
\begin{tikzpicture}
\begin{axis}[
width=4.4cm,
height=4.4cm,
ybar,
bar width=0.13cm,
grid=major,
tick pos=left,
tick label style={font=\scriptsize},
xtick={-8, -6, -4, -2, 0, 2, 4, 6, 8},
xticklabels={-8, -6, -4, -2, 0, 2, 4, 6, 8},
ytick={0, 0.1, 0.2, 0.4, 0.5, 0.6},
yticklabels={0, 1e-2, 2e-2, 0.3, 0.4, 0.5},
xlabel={Change in pixel value \\},
ylabel={Ratio},
xlabel style={font=\scriptsize, at={(0.5, -0.02)}, align=center},
ylabel near ticks,
ylabel style={font=\scriptsize, at={(-0.17, 0.5)}},
ymin=0,
ymax=0.6
]
\addplot coordinates
{(-8, 0.5478980457715396) (-6, 0.14003602798375706) (-4, 0.1750571040783898)  (-2, 0.0639938095868644) (0, 0.1985194319385593) 
(2, 0.0743456479519774) (4, 0.18316064177259883) (6, 0.14652575476694912) (8, 0.5539381124205508)};
\end{axis}
\node[font=\tiny] at (0.24, -0.59) {(-$\epsilon$)};
\node[font=\tiny] at (2.58, -0.59) {($\epsilon$)};
\draw (0, 1.43) -- node[fill=white,inner sep=-1.25pt,outer sep=0,anchor=center]{$\approx$} (0, 1.43);
\draw (2.82, 1.43) -- node[fill=white,inner sep=-1.25pt,outer sep=0,anchor=center]{$\approx$} (2.82, 1.43);
\end{tikzpicture}

\vspace{-0.5em}
\caption{Distribution of adversarial noise with white box PGD attack on Cifar-10 dataset with wide Resnet w32-10 adversarially trained network at $\ell_\infty$ ball radius $\epsilon=8$ in [0, 255] scale.}
\label{fig:noise_distribution}
\end{minipage}
\hspace{1em}
\begin{minipage}[t]{0.46\columnwidth}
\begin{tikzpicture}
\begin{axis}[
width=4.4cm,
height=4.4cm,
grid=major,
tick pos=left,
tick label style={font=\scriptsize},
ytick={-0.6, -0.4, -0.2, 0, 0.2, 0.4, 0.6},
yticklabels={-0.6, -0.4, -0.2, 0, 0.2, 0.4, 0.6},
xlabel={Image index \\ (sorted by $f(\cdot)-f(x_{pgd})$)},
xlabel style={font=\scriptsize, align=center},
ylabel={$f(\cdot)-f(x_{pgd})$},
ylabel near ticks,
ylabel style={font=\scriptsize, at={(-0.15,0.5)}},
xmin=0,
xmax=100,
legend style={nodes={scale=0.7, transform shape}},
no marks,
]
\addplot+[const plot] table [x=x, y=ls, col sep=comma] {data/fig2_diff.csv};
\addlegendentry{$x_{ls}$}
\addplot+[const plot] table [x=x, y=g, col sep=comma] {data/fig2_diff.csv};
\addlegendentry{$x_{g}$}
\end{axis}
\end{tikzpicture}

\vspace{-0.5em}
\caption{$f(\cdot)-f(x_{pgd})$ values on random 100 samples on the same experiment setting as \Cref{fig:noise_distribution}, where $x_{ls}$, $x_{g}$ and $x_{pgd}$ each denotes image perturbed by local search, greedy insertion, and PGD method.}
\label{fig:submodularity}
\end{minipage}
\end{figure}

Intuitively, submodular functions exhibit a diminishing returns property where the marginal gain diminishes as the augmenting set size increases. In the context of machine learning problems, the implication of submodularity is that we can efficiently compute an approximately optimal solution to the submodular set functions with a suite of greedy style algorithms. Furthermore, submodularity allows us to establish $(1-\frac{1}{e})$-approximation \cite{nemhauser78} for monotone submodular functions and $\frac{1}{3}$-approximation \cite{feige07} for non-monotone submodular functions.


Unfortunately, we can construct a counterexample showing $F(\Scal)$ is not submodular. Let $f(x)=-\log(\frac{\mathrm{1} }{\mathrm{1} + e^{-w^\intercal x} })$, $w=(-1, -1)^\intercal$, $\epsilon=1$, and $x=(0, 0)^\intercal$. For $\Acal=\emptyset, \: \Bcal=\{1\}$, and $e=2$ the inequality on Definition 2 does not hold because $0.57=\Delta(e \mid \Acal) < \Delta (e \mid \Bcal)=1.43$, which implies $F(\Scal)$ is not submodular. Regardless, if the submodularity is not severely deteriorated \cite{zhou16}, submodular maximization algorithms still work to a substantial extent. \Cref{fig:submodularity} shows the result of running local search \cite{feige07} and greedy insertion \cite{krause12} in comparison to the function value of the solution obtained by white box PGD method on Cifar-10 dataset. The plots show that greedy insertion solution, in general, has comparable objective value to PGD solution and that local search solution shows slightly higher value than PGD solution.

For non-monotone set functions, local search type algorithms achieve better theoretical and practical solutions than greedy insertion algorithms by providing modifications to the working set with alternating insertion and deletion processes. In the following subsection, we establish the approximation bound for performing improved local search procedure for non-monotone approximate submodular functions and then propose an accelerated variant of the algorithm for practical black-box adversarial attack. We first make a slight detour and introduce the local search procedure and a proof of local optimality.


%
%
%
%

\subsection{Local search optimization for black-box attack on approximately submodular functions}

Local search algorithm \cite{feige07} alternates between greedily inserting an element while the marginal gain is strictly positive ($\Delta(e \mid \Scal) > 0$) and removing an element while the marginal gain is also strictly positive. \citet{feige07} showed that once the algorithm converges it converges to a \emph{local optimum of the set function} $F$. When the set function is submodular, the local search solution has $\frac{1}{3}$-approximation with respect to the optimal solution. 

\begin{definition}
Given a set function $F$, a set $\Scal$ is a local optimum, if $F(\Scal)\ge F(\Scal \setminus \{a\})$ for any $a \in \Scal$ and $F(\Scal)\ge F(\Scal \cup \{a\})$ for any $a \notin \Scal$.
\end{definition}

\begin{lemma}
\label{lem:local_opt}
Let $\Scal$ be the solution obtained by performing the local search algorithm. Then $\Scal$ is a local optima.
\end{lemma}

\begin{proof}
See supplementary A.1.
\end{proof}

%

Note that \Cref{lem:local_opt} holds regardless of submodularity. The following theorem states an approximation bound for the local search solution for approximately submodular set functions. Note, we assume non-negativity\footnote{A standard trick in discrete optimization is to add a constant offset term to the set function.} of the set function for the proof.

\begin{theorem}
Let $\Ccal$ be an optimal solution for a function $F$ and $\Scal$ be the solution obtained by the local search algorithm. Then,
\begin{align*}
    2F(\Scal)+F(\Vcal \setminus \Scal) \ge  F(\Ccal) + \xi \lambda_F(\Vcal, 2),
\end{align*}
where 
\small
\begin{align*}
    \xi = \binom{|\Scal \setminus \Ccal|}{2} + \binom{|\Ccal \setminus \Scal|}{2} + |\overline{\Scal \cup \Ccal}| \cdot |\Scal| + |\Ccal \setminus \Scal| \cdot |\Scal \cap \Ccal|
\end{align*}
\label{thm:main}
\end{theorem}
\normalsize

\begin{proof} See supplementary A.2. \end{proof}

Finally, from \Cref{thm:main}, we get the following corollary.

\begin{corollary}
If $F(\Ccal) + \xi \lambda_F(\Vcal, 2) \ge 0$, one of the following holds
\begin{enumerate}
    \item $F(\Scal) \ge \frac{1}{3} \Big(F(\Ccal) + \xi \lambda_F(\Vcal, 2)\Big)$
    \item $F(\Vcal \setminus \Scal) \ge \frac{1}{3} \Big(F(\Ccal) + \xi \lambda_F(\Vcal, 2)\Big)$
\end{enumerate}
\end{corollary}

As the set function becomes submodular, the submodularity index ($\lambda_F$) becomes close to zero \cite{zhou16}, recovering the $\frac{1}{3}$-approximation bound \cite{feige07}. From the local search algorithm, we obtain a set $\Scal$ of pixels to perturb the input image $x$ with $+\epsilon$ and the complement set $\Vcal \setminus \Scal$ of pixels to perturb with $-\epsilon$. Concretely, the perturbed image is computed as $x_{adv} \triangleq x + \epsilon \sum_{i \in \Scal}e_i - \epsilon \sum_{i \notin \Scal} e_i$. \Cref{fig:adv_image_examples} shows some examples of the perturbed images produced by the algorithm.

\begin{figure}[ht]
\centering
\includegraphics[width=\columnwidth]{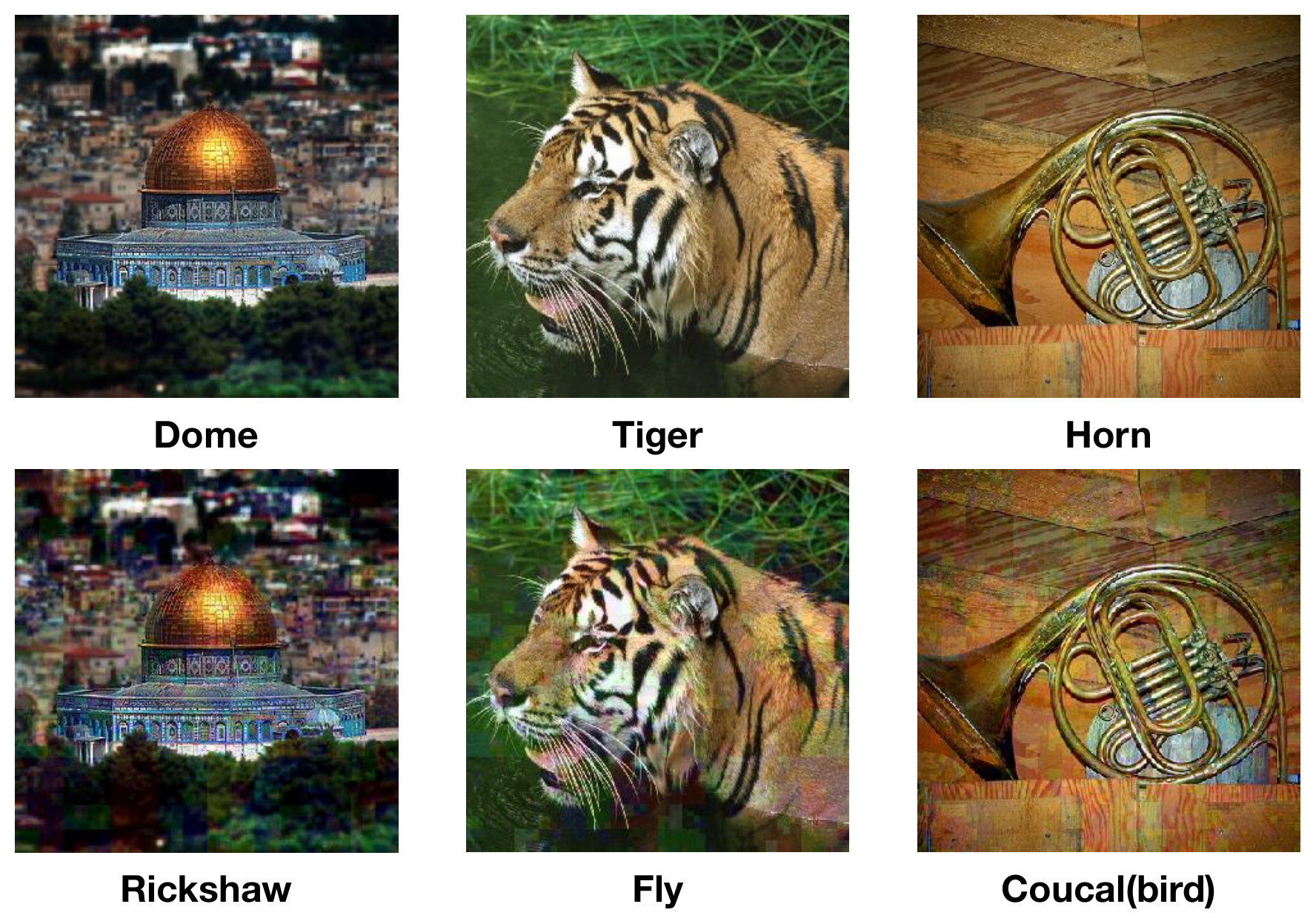}

\vspace{-0.5em}
\caption{Adversarial examples from ImageNet in the targeted setting. The top row shows the original images and the bottom row shows the corresponding perturbed images from our method.}
\label{fig:adv_image_examples}
\end{figure}

\subsection{Acceleration with lazy evaluations} \label{sec:local_search}
Na\"ively applying the local search algorithm \cite{feige07} for black-box adversarial attack poses a challenge because each calls to the greedy insertion or deletion algorithms make $O(|\Vcal| \cdot |\Scal|)$ queries (since at each step greedy algorithm finds the element that maximizes the marginal gain) and would become impractical for query limited black-box attacks which we are interested in. 

\begin{algorithm}[H]
\caption{Lazy Greedy Insertion}
\label{alg:lazy_insert}
\small
\begin{algorithmic}[1]
\INPUT Objective set function $F$, Working set $\Scal$, Ground set $\Vcal$
\REQUIRE Max heap $Q=\emptyset$ 
\FOR{each element $e \in \Vcal \setminus \Scal$}
\STATE Calculate $\Delta(e \mid \Scal) := F(\Scal \cup \{e\})-F(\Scal)$
\STATE Push $(e, \Delta(e \mid \Scal))$ into $Q$
\ENDFOR
\WHILE{$|Q| > 0$}
\STATE Pop the top element $\hat{e}$ from $Q$, update its upper bound $\rho(\hat{e})$ 
\STATE Peek the top element $\tilde{e}$ and its upper bound $\rho(\tilde{e})$ in $Q$
\IF{$\rho(\hat{e}) > \rho(\tilde{e})$ and $\rho(\hat{e}) > 0$}
\STATE $\Scal \leftarrow \Scal \cup \{\hat{e}\}$   
\ELSIF{$\rho(\hat{e}) > \rho(\tilde{e})$ and $\rho(\hat{e}) \le 0$}
\STATE \textbf{break}
\ELSE
\STATE Push $(\hat{e}, \rho(\hat{e}))$ into $Q$
\ENDIF                                   
\ENDWHILE                                    
\OUTPUT $\Scal$;
\end{algorithmic}
\end{algorithm}

\vspace{-1em}

Thus, we employ the accelerated greedy algorithm often called Lazy-Greedy \cite{minoux78}. Instead of computing the marginal gain $\Delta(e \mid \Scal_i)$ for each item $e \in \Vcal \setminus \Scal_i$ at each iteration $i+1$, the algorithm keeps an upper bound $\rho(e)$ on the marginal gain for each item in a max-heap. In each iteration $i+1$, it evaluates the marginal gain for the top element $\hat{e}$ in the heap and updates the upper bound $\rho(\hat{e}) := \Delta(\hat{e} \mid \Scal_i)$. If $\rho(\hat{e}) \geq \rho(e)~~ \forall e$, then submodularity guarantees that $\hat{e}$ is the element with the largest marginal gain. If the top element does not satisfy this condition,  the algorithm inserts it again into the heap with the updated upper bound. 

\begin{figure*}[ht]
\centering
\includegraphics[width=1\textwidth]{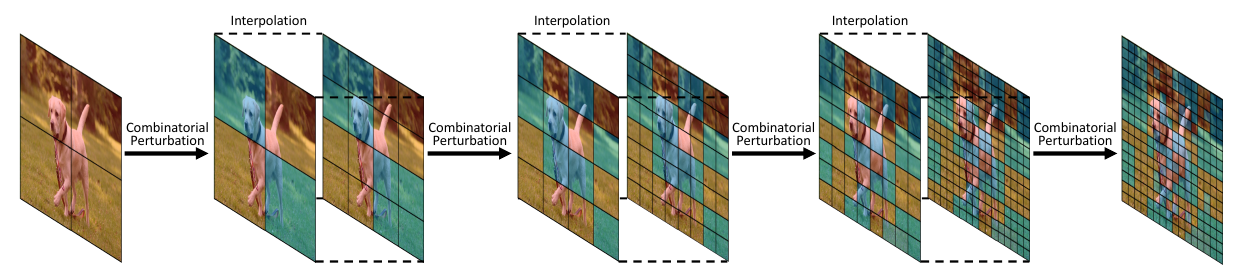}

\vspace{-0.5em}
\caption{Illustration of hierarchical lazy evaluation process. Blue area represents the blocks currently in the working set ($\Scal$), while red area represents blocks outside the working set ($\Vcal \setminus \Scal$).}
\label{fig:hierarchy}
\end{figure*}

While the number of function evaluations in the worst case is the same as the standard greedy algorithm, the algorithm provides \emph{several orders of magnitude speedups} in practice \cite{leskovec07, rodriguez12, linbilmes12, wei13, mirzasoleimane14}. Note, however, if strict upper bound on function evaluation is required, another variant, Stochastic-Greedy \cite{lazier_than_lazy_greedy} offers a stochastic algorithm with linear time upper bound on function evaluations ($|\Vcal| \log \frac{1}{\theta}$).

\begin{algorithm}[H]
\caption{Lazy Greedy Deletion}
\label{alg:lazy_remove}
\small
\begin{algorithmic}[1]
\INPUT Objective set function $F$, Working set $\Scal$
\REQUIRE Max heap $Q=\emptyset$ 
\FOR{each element $e \in \Scal$}
\STATE Calculate $\Delta^{-}(e \mid \Scal) := F(\Scal \setminus \{e\})-F(\Scal)$ 
\STATE Push $(e, \Delta^{-}(e \mid S))$ into $Q$
\ENDFOR
\WHILE{$|Q| > 0$}
\STATE Pop the top element $\hat{e}$ from $Q$, update its upper bound $\rho(\hat{e})$
\STATE Peek the top element $\tilde{e}$ and its upper bound $\rho(\tilde{e})$ in $Q$
\IF{$\rho(\hat{e}) > \rho(\tilde{e})$ and $\rho(\hat{e}) > 0$}
\STATE $\Scal \leftarrow \Scal \setminus \{\hat{e}\}$  
\ELSIF{$\rho(\hat{e}) > \rho(\tilde{e})$ and $\rho(\hat{e}) \le 0$}
\STATE \textbf{break}                                
\ELSE
\STATE Push $(\hat{e}, \rho(\hat{e}))$ into $Q$
\ENDIF                                   
\ENDWHILE                                    
\OUTPUT $\Scal$;
\end{algorithmic}
\end{algorithm}

\vspace{-1em}

\begin{algorithm}[H]
\caption{Accelerated Local Search w/ Lazy Evaluations}
\label{alg:local_search}
\small
\begin{algorithmic}[1]
\INPUT Objective set function $F$, Working set $\Scal$, Ground set $\Vcal$
\FOR{$t=1,\ldots, \text{MAXITER}$}
\STATE Insert elements of $\Vcal$ into $\Scal$ using \Cref{alg:lazy_insert}\\
$\Scal \leftarrow \textsc{LazyGreedyInsertion}(F, \Scal, \Vcal)$
\STATE Delete elements from $\Scal$ using \Cref{alg:lazy_remove}\\
$\Scal \leftarrow \textsc{LazyGreedyDeletion}(F, \Scal)$
\ENDFOR
\OUTPUT $\argmax\limits_{\Acal \in \{\Scal, \Vcal \setminus \Scal \}} F(\Acal)$;
\end{algorithmic}
\end{algorithm}

\vspace{-1.5em}

The resulting algorithm for performing local search optimization with lazy evaluations is presented in \Cref{alg:local_search}. The lazy insertion and deletion algorithms are presented in \Cref{alg:lazy_insert} and \Cref{alg:lazy_remove} respectively.

\subsection{Hierarchical lazy evaluation}
Most of the recent black-box approaches \cite{zoo, autozoom, bandit} exploit the spatial regularities inherent in images for query efficiency. The underlying idea is that natural images exhibit locally regular structure \cite{huangMumfordCvpr1999} and is far from random matrices of numbers. \citet{autozoom, bandit} exploit the idea to estimate the gradient in blocks of pixels and perform interpolation to compute the full gradient image.

\begin{algorithm}[H]
\caption{Split Block}
\label{alg:split_block}
\small
\begin{algorithmic}[1]
\INPUT Set of blocks $\Bcal$, Block size $k$
\REQUIRE $\Bcal' = \emptyset$
\FOR{each block $b \in \Bcal$}
\STATE Split the block $b$ into 4 blocks $\{b_1, b_2, b_3, b_4\}$ with size $k/2$
\STATE $\Bcal' \leftarrow \Bcal' \cup \{b_1, b_2, b_3, b_4\}$
\ENDFOR
\OUTPUT $\Bcal'$;
\end{algorithmic}
\end{algorithm}

\vspace{-1em}

\begin{algorithm}[H]
\caption{Hierarchical Accelerated Local Search}
\label{alg:final_alg}
\small
\begin{algorithmic}[1]
\INPUT Objective set function $F$, Block size $k$, Ground set $\Vcal$ of size $| \Vcal | = h/k \times w/k \times c$~ where the image size is $h \times w \times c$
\REQUIRE Working set $\Scal = \emptyset$
\REPEAT
\STATE Run Local Search Algorithm on $\Scal$ and $\Vcal$ using \Cref{alg:local_search} \\
$\Scal \leftarrow \textsc{LocalSearch}(F, \Scal, \Vcal)$
\IF{$k >1$}
\STATE Split the blocks into finer blocks using \Cref{alg:split_block} \\
$\Scal \leftarrow \textsc{SplitBlock}(\Scal, k)$, $\Vcal \leftarrow \textsc{SplitBlock}(\Vcal, k)$
\STATE $k \leftarrow k / 2$
\ENDIF
\UNTIL{$F$ converges}
\OUTPUT $\Scal$;
\end{algorithmic}
\end{algorithm}

\vspace{-1em}

We take a hierarchical approach and perform the accelerated local search in \Cref{alg:local_search} on a coarse grid (large blocks) and use the results to define the initial working set in the subsequent rounds on finer grid (smaller blocks) structure. \Cref{fig:hierarchy} illustrates the process for several rounds. \Cref{alg:final_alg} shows the overall hierarchical accelerated local search algorithm. The ground set $\Vcal$ now becomes the set of all blocks, not pixels. It is important to note that most of the attacks terminate in early stages and rarely run until the very fine scales. Even in the case when the algorithm proceeds into finer granularities, the algorithm pre-terminates when it reaches the maximum allowed query limit following the experimental protocol in \citet{nes, bandit}.

\section{Implementation details}\label{sec:impl}

We assume only the cross-entropy loss is available by model access, following \citet{nes, bandit}. On Cifar-10, We set the initial block size to $k=4$. On ImageNet, we set the initial block size to $k=32$. To make blocks divisible by 2, we set the noise size to be $256\times 256$ and resize it to the image size ($299\times 299$) using nearest neighbor interpolation. Since our method runs in query-limited setting, we fix $\text{MAXITER}$ (in \Cref{alg:local_search}) to 1, reducing unnecessary query counts for calculating marginal gains.

Our method needs $O(\frac{|\Vcal|}{k^2})$ queries for calculating the initial marginal gains. This can consume excessive amount of queries before the actual perturbation. To address this, instead of running the algorithm on the whole ground set $\Vcal$, we split $\Vcal$ into a partition of mini-batches $\{\Vcal_i\}_{i=1}^{n}$ and split the working set $\Scal$ into $\{\Scal_i\}_{i=1}^{n}$ where $\Scal_i = \Scal \cap \Vcal_i$. Then we update $\Scal_i$ subject to $\Vcal_i$ for $i=1, ..., n$ sequentially. Concretely, at $i$-th step we insert elements of $\Vcal_i$ into $\Scal_i$ and remove the elements from the updated $\Scal_i$. We fix the mini-batch size to 64 throughout all the experiments.

\section{Experiments}\label{sec:exp}

We evaluate the performance comparing against the NES method \cite{nes} and the Bandits method \cite{bandit}, which is the current state of the art in black-box attacks, on both untargeted and targeted attack settings. We consider the $\ell_\infty$ threat models on Cifar-10 and ImageNet datasets and quantify the performance in terms of success rate, average queries, and median queries. We further investigate the average queries on samples that NES, the weakest attack among the baselines, successfully fooled. Note that this is a fairer measure for evaluating an attack method's performance, since na\"ive average query measure can be affected by the method's success rate. More accurate methods can be disadvantaged by successfully fooling more difficult images slightly below the max query budget. We also show the white-box PGD results from \citet{pgd} as the \emph{upper bound} experiment.

\begin{figure*}[ht]
	\centering
	\begin{subfigure}[t]{0.3\textwidth}
		\begin{tikzpicture}
		\begin{axis}[
		width=5cm,
		height=4.8cm,
		no marks,
		every axis plot/.append style={thick},
		grid=major,
		scaled ticks = false,
		ylabel near ticks,
		tick pos=left,
		tick label style={font=\small},
		xtick={0, 4000, 8000, 12000, 16000, 20000},
		xticklabels={0, 4k, 8k, 12k, 16k, 20k},
		ytick={0, 10, 20, 30, 40, 50},
		yticklabels={0, 10, 20, 30, 40, 50},
		label style={font=\small},
		xlabel={The number of queries},
		ylabel={Success rate},
		xmin=0,
		xmax=20000,
		ymin=0,
		ymax=55,
		]
		\addplot[red] table [x=base, y=Ours, col sep=comma]{data/cifar10_untargeted.csv};
		\addplot[brown] table [x=base, y=NES, col sep=comma]{data/cifar10_untargeted.csv};
		\addplot[blue] table [x=base, y=Bandits, col sep=comma]{data/cifar10_untargeted.csv};
		\draw[dashed] (0, 472) -- (20000, 472);
		\end{axis}
		\end{tikzpicture}
		
		\vspace{-0.2em}
		\caption{Cifar-10, untargeted}\label{fig:cdf_cifar10}
	\end{subfigure}
	\begin{subfigure}[t]{0.3\textwidth}
		\begin{tikzpicture}
		\begin{axis}[
		width=5cm,
		height=4.8cm,
		no marks,
		every axis plot/.append style={thick},
		grid=major,
		scaled ticks = false,
		ylabel near ticks,
		tick pos=left,
		tick label style={font=\small},
		xtick={0, 2000, 4000, 6000, 8000, 10000},
		xticklabels={0, 2k, 4k, 6k, 8k, 10k},
		ytick={0, 20, 40, 60, 80, 100},
		yticklabels={0, 20, 40, 60, 80, 100},
		label style={font=\small},
		xlabel={The number of queries},
		ylabel={Success rate},
		ylabel style={at={(-0.2,0.5)}},
		xmin=0,
		xmax=10000,
		ymin=0,
		ymax=110,
		legend style={legend columns=3, at={(1.22, 1.22)}},
		]
		\addplot[red] table [x=base, y=Ours, col sep=comma]{data/imagenet_untargeted.csv};
		\addlegendentry{Ours}
		\addplot[brown] table [x=base, y=NES, col sep=comma]{data/imagenet_untargeted.csv};
		\addlegendentry{NES}
		\addplot[blue] table [x=base, y=Bandits, col sep=comma]{data/imagenet_untargeted.csv};
		\addlegendentry{Bandits}
		\draw[dashed] (0, 99.9) -- (10000, 99.9);
		\end{axis}
		\end{tikzpicture}
		
		\vspace{-0.2em}
		\caption{ImageNet, untargeted}\label{fig:cdf_imagenet_untargeted}
	\end{subfigure}
	\begin{subfigure}[t]{0.33\textwidth}
		\begin{tikzpicture}
		\begin{axis}[
		width=5cm,
		height=4.8cm,
		no marks,
		every axis plot/.append style={thick},
		grid=major,
		scaled ticks = false,
		ylabel near ticks,
		tick pos=left,
		tick label style={font=\small},
		xtick={0, 20000, 40000, 60000, 80000, 100000},
		xticklabels={0, 20k, 40k, 60k, 80k, 100k},
		ytick={0, 20, 40, 60, 80, 100},
		yticklabels={0, 20, 40, 60, 80, 100},
		label style={font=\small},
		xlabel={The number of queries},
		ylabel={Success rate},
		ylabel style={at={(-0.2,0.5)}},
		xmin=0,
		xmax=100000,
		ymin=0,
		ymax=110,
		]
		\addplot[red] table [x=base, y=Ours, col sep=comma]{data/imagenet_targeted.csv};
		\addplot[brown] table [x=base, y=NES, col sep=comma]{data/imagenet_targeted.csv};
		\addplot[blue] table [x=base, y=Bandits, col sep=comma]{data/imagenet_targeted.csv};
		\draw[dashed] (0, 99.9) -- (10000, 99.9);
		\end{axis}
		\end{tikzpicture}
		
		\vspace{-0.2em}
		\caption{ImageNet, targeted}\label{fig:cdf_imagenet_targeted}
	\end{subfigure}
	\vspace{-0.2em}
	\caption{The cumulative distribution of the number of queries required for (a) untargeted attack on Cifar-10, (b) untargeted attack on ImageNet, and (c) targeted attack on ImageNet. The dashed line indicates the success rate of white-box PGD. The results show that our method consistently finds successful adversarial images faster than the baseline methods.}
	\label{fig:cdf}
\end{figure*}
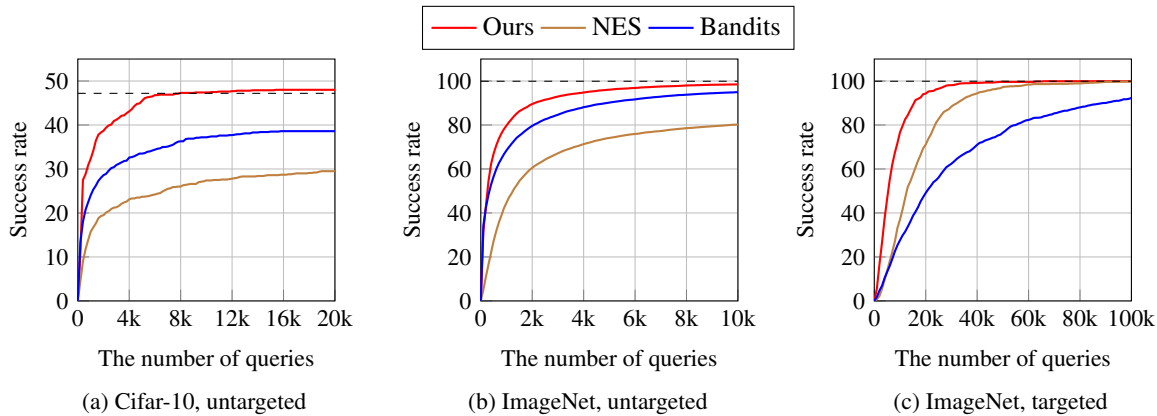

\subsection{Experiments on Cifar-10}

To evaluate the effectiveness of the method on the adversarially trained network, which is known to be robust to adversarial perturbations, we tested the attacks on wide Resnet w32-10 classifier \cite{zagoruyko16} adversarially trained on Cifar-10 dataset \cite{pgd}. We use the pretrained network provided by MadryLab\footnote{\url{https://github.com/MadryLab/cifar10_challenge}}. We then use 1,000 randomly selected images from the validation set that are initially correctly classified. We set the maximum distortion of the adversarial image to $\epsilon=8$ in $[0, 255]$ scale, following the experimental protocol in \citet{pgd, athalye18, bhagoji18}. We restrict the maximum number of queries to 20,000.

We run 20 iterations of PGD with constant step size of 2.0, as done in \citet{pgd}. We performed grid search for hyperparameters in NES and Bandits. For NES, we tuned $\sigma \in \{0.0001, 0.001, 0.01\}$, size of NES population $n \in \{50, 100, 200\}$, learning rate $\eta \in \{0.001, 0.005, 0.01\}$, and momentum $\beta \in \{0.1, 0.3, 0.5, 0.7, 0.9\}$. For Bandits, we tuned for OCO learning rate $\eta \in \{0.01, 0.1, 1, 10, 100\}$, image learning rate $h \in \{0.001, 0.005, 0.01\}$, bandit exploration $\delta \in \{0.01, 0.1, 1\}$, and finite difference probe $\eta \in \{0.01, 0.1, 1\}$. The hyperparameters used are listed in supplementary B.1.

The results are presented in \Cref{tab:cifar10_untargeted} and \Cref{fig:cdf_cifar10}. We found that our algorithm has about 10\% higher success rate than Bandits, with 33\% less average queries. Notably, this success rate is higher than the white-box PGD method. The efficiency of our algorithm is more evident in average queries on samples that NES successfully fooled. On this measure, our method needs 53\% less queries on average compared to Bandits, and 91\% less compared to NES.

\begin{table}[H]
	\centering
	\small
	\begin{adjustbox}{max width=\columnwidth}
		\begin{tabular}{>{\centering}m{1.9cm} >{\centering}m{1cm} >{\centering}m{1cm} >{\centering}m{1cm} | >{\centering}m{1.9cm}}
			\toprule[1pt]
			\textbf{Method} & \textbf{Success rate} & \textbf{Avg. queries}  & \textbf{Med. queries} & \textbf{Avg. queries ~~~~ (NES success)} \\
			\midrule
			PGD \scriptsize(white-box) & 47.2\% & 20 & - & - \\
			\midrule
			NES & 29.5\% & 2872 & 900 & 2872 \\
			Bandits & 38.6\% & 1877 & 459 & 520 \\
			\textbf{Ours} & \textbf{48.0}\% & \textbf{1261} & \textbf{356} & \textbf{247} \\
			\bottomrule[1pt]
		\end{tabular}
	\end{adjustbox}
	
	\vspace{-0.2em}
	\caption{Results for $\ell_\infty$ untargeted attacks on Cifar-10. Maximum number of queries set to 20,000.}
	\label{tab:cifar10_untargeted}
\end{table}

\vspace{-0.5em}

\subsection{Untargeted attacks on ImageNet}

On ImageNet, we attack the pretrained Inception v3 classifier from \citet{inceptionv3} provided by Tensorflow\footnote{\url{https://github.com/tensorflow/models/tree/master/research/slim}}. We use 10,000 randomly selected images (scaled to [0, 1]) that are initially correctly classified. We set $\epsilon$ to 0.05 and the maximum queries to 10,000, as done in \citet{bandit}.

We ran PGD for 20 steps at the constant step size of 0.01. We communicated with the authors of \citet{bandit}, and the authors provided the up-to-date hyperparameters for Bandits. Hyperparameters for NES were referred from the paper. NES\textsuperscript{\textdagger} and Bandits\textsuperscript{\textdagger} denote the results copied from the paper \cite{bandit} for comparison. The hyperparameters used are listed in supplementary B.2.

The results are presented in \Cref{tab:imagenet_untargeted} and \Cref{fig:cdf_imagenet_untargeted}. We found that our method again outperforms other black-box methods by a significant margin. Our method achieves about 4\% higher success rate than Bandits, with 30\% less queries. Also, note that our method requires 38\% less average queries on samples that NES successfully attacked than Bandits.

\begin{table}[H]
	\centering
	\small
	\begin{adjustbox}{max width=\columnwidth}
		\begin{tabular}{>{\centering}m{1.9cm} >{\centering}m{1cm} >{\centering}m{1cm} >{\centering}m{1cm} | >{\centering}m{1.9cm}}
			\toprule
			\textbf{Method} & \textbf{Success rate} & \textbf{Avg. queries}  & \textbf{Med. queries} & \textbf{\shortstack{Avg. queries \\ (NES success)}} \\
			\midrule
			PGD \scriptsize{(white-box)} & 99.9 \% & 20 & - & - \\
			\midrule
			NES\textsuperscript{\textdagger} & 77.8\% & 1735 & - &1735 \\
			NES & 80.3\% & 1660 & 900 &1660 \\
			Bandits\textsuperscript{\textdagger} & 95.4\% & 1117 & - &703 \\
			Bandits & 94.9\% & 1030 & 286 & 603 \\
			\textbf{Ours} & \textbf{98.5\%} & \textbf{722} & \textbf{237} & \textbf{376} \\
			\bottomrule[1pt]
		\end{tabular}
	\end{adjustbox}
	
	\vspace{-0.2em}
	\caption{Results for $\ell_\infty$ untargeted attacks on ImageNet. Maximum number of queries set to 10,000.}
	\label{tab:imagenet_untargeted}
\end{table}

\vspace{-0.5em}

\subsection{Targeted attacks on ImageNet}

For targeted attacks, we use the same Inception v3 network used in the untargeted attack setting. We attack 1,000 randomly selected images (scaled to [0, 1]) that are initially correctly classified. Targeted classes were chosen randomly for each image, and each attack method chose the same target classes for the same images for fair comparison. We limit $\epsilon$ to 0.05 and the maximum number of queries to 100,000.

We ran PGD for 200 steps at the constant step size of 0.001. Hyperparameters for NES were adjusted from \citet{nes} to match the results in the paper (given in supplementary B.3). Since \citet{bandit} does not report targeted attack for Bandits, we performed hyperparameter tuning for Bandits. We tuned for the image learning rate $h \in\{0.0001, 0.001, 0.005, 0.01, 0.05\}$ and OCO learning rate $\eta \in\{1, 10, 100, 1000\}$. Details of the experiment's results can be found in supplementary C. NES\textsuperscript{\textdagger}\tablefootnote{In the original paper, the query limit was 1,000,000.} indicates the result reported in the paper \cite{nes} for comparison. The results are presented in \Cref{tab:imagenet_targeted} and \Cref{fig:cdf_imagenet_targeted}. We can see that our method achieves a higher success rate (near 100\%), with about 55\% less queries than NES. 

\begin{table}[H]
	\centering
	\small
	\begin{adjustbox}{max width=\columnwidth}
		\begin{tabular}{>{\centering}m{1.9cm} >{\centering}m{1cm} >{\centering}m{1cm} >{\centering}m{1cm} | >{\centering}m{1.9cm}}
			\toprule[1pt]
			\textbf{Method} & \textbf{Success rate} & \textbf{Avg. queries} &  \textbf{Med. queries} & \textbf{Avg. queries (NES success)} \\
			\midrule
			PGD \scriptsize(white-box) & 100\% & 200 & - & - \\
			\midrule
			NES\textsuperscript{\textdagger}& 99.2\% & - & 11550 & - \\
			NES & 99.7\% & 16284 & 12650 & 16284 \\
			Bandits & 92.3\%  & 26421 & 18642 & 26421 \\
			\textbf{Ours} & \textbf{99.9\%} & \textbf{7485} & \textbf{5373}& \textbf{7371} \\
			\bottomrule[1pt]
		\end{tabular}
	\end{adjustbox}
	
	\vspace{-0.2em}
	\caption{Results for $\ell_\infty$ targeted attacks on ImageNet. Maximum number of queries set to 100,000.}
	\label{tab:imagenet_targeted}
\end{table}

\vspace{-0.5em}

\begin{figure*}[ht]
	\begin{minipage}[t]{0.48\textwidth}
		\begin{tikzpicture}
		\begin{groupplot}[
		group style = {group size = 2 by 1, horizontal sep=1.1cm, vertical sep=0cm},
		width = 4.6cm,
		height = 4.6cm,
		no marks,
		every axis plot/.append style={thick},
		grid=major,
		scaled ticks = false,
		tick pos = left,
		tick label style={font=\small},
		xtick={0, 2000, 4000, 6000, 8000, 10000},
		xticklabels={0, 2k, 4k, 6k, 8k, 10k},
		ytick={0, 20, 40, 60, 80, 100},
		yticklabels={0, 20, 40, 60, 80, 100},
		xmin=0,
		xmax=10000,
		ymin=0,
		ymax=110,
		xlabel={The number of queries},
		ylabel={Success rate},
		label style={font=\small},
		ylabel near ticks,
		ylabel style={at={(-0.15,0.5)}},
		]
		\nextgroupplot[
		legend style={legend columns=3, font=\small},
		legend to name=attack,
		]
		\addplot[red] table [x=base, y=Ours, col sep=comma]{data/imagenet_untargeted_0.01.csv};
		\addlegendentry{Ours}
		\addplot[brown] table [x=base, y=NES, col sep=comma]{data/imagenet_untargeted_0.01.csv};
		\addlegendentry{NES}
		\addplot[blue] table [x=base, y=Bandits, col sep=comma]{data/imagenet_untargeted_0.01.csv};
		\addlegendentry{Bandits}
		\draw[dashed] (0, 99.5) -- (10000, 99.5);
		\coordinate (c1) at (rel axis cs:0,1);
		\nextgroupplot[
		]
		\addplot[red] table [x=base, y=Ours, col sep=comma]{data/imagenet_untargeted_0.03.csv};
		\addplot[brown] table [x=base, y=NES, col sep=comma]{data/imagenet_untargeted_0.03.csv};
		\addplot[blue] table [x=base, y=Bandits, col sep=comma]{data/imagenet_untargeted_0.03.csv};
		\draw[dashed] (0, 100) -- (10000, 100);
		\coordinate (c2) at (rel axis cs:1,0);
		\end{groupplot}
		\coordinate (c3) at ($(c1)!.5!(c2)$);
		\node[above] at (c3 |- current bounding box.north) {\pgfplotslegendfromname{attack}};
		\end{tikzpicture}
		
		\vspace{-0.2em}
		\caption{The cumulative distribution for the number of queries required for untargeted attack on ImageNet with $\epsilon=0.01$  (left) and  $\epsilon=0.03$ (right).}
		\label{fig:cumulative_distribution_epsilon}
	\end{minipage}
	\hspace{0.5cm}
	\begin{minipage}[t]{0.48\textwidth}
		\centering
		\begin{tikzpicture}
		\begin{axis}[
		width = 5.8cm,
		height = 5.36cm,
		only marks,
		grid=major,
		scaled ticks = false,
		ylabel near ticks,
		tick pos = left,
		tick label style={font=\small},
		xtick={600, 800, 1000, 1200, 1400},
		xticklabels={600, 800, 1000, 1200, 1400},
		ytick={92, 94, 96, 98, 100},
		ytick={92, 94, 96, 98, 100},
		label style={font=\small},
		xlabel={Average queries},
		ylabel={Success rate},
		xlabel style={at={(0.5, 0.02)}},
		xmin=600,
		xmax=1400,
		ymin=91,
		ymax=100,
		legend style={legend columns=1, font=\scriptsize},
		legend cell align={left},
		legend to name=attack,
		]
		\addplot+[red, mark options={fill=red, scale=0.5, mark=*, solid}] table  [x=avg queries, y=success rate, col sep=comma]{data/ablation_lls.csv};
		\addlegendentry{Ours}
		\node[font=\tiny] at (120,680) {32};
		\node[font=\tiny] at (40,740) {64};
		\node[font=\tiny] at (130, 760) {128};
		\addplot+[magenta, mark options={fill=magenta, scale=0.5, mark=square*, solid}] table  [x=avg queries, y=success rate, col sep=comma]{data/ablation_bandit_exploration_0.1_fd_eta_0.01.csv};
		\addlegendentry{$\delta=0.1$, $\eta=0.01$}
		\addplot+[brown, mark options={fill=brown, scale=0.5, mark=square*, solid}] table  [x=avg queries, y=success rate, col sep=comma]{data/ablation_bandit_exploration_0.1_fd_eta_0.1.csv};
		\addlegendentry{$\delta=0.1$, $\eta=0.1$}
		\addplot+[pink, mark options={fill=pink, scale=0.5, mark=square*, solid}] table  [x=avg queries, y=success rate, col sep=comma]{data/ablation_bandit_exploration_0.1_fd_eta_1.0.csv};
		\addlegendentry{$\delta=0.1$, $\eta=1$}
		\addplot+[orange, mark options={fill=orange, scale=0.5, mark=square*, solid}] table  [x=avg queries, y=success rate, col sep=comma]{data/ablation_bandit_exploration_1.0_fd_eta_0.01.csv};
		\addlegendentry{$\delta=1$, $\eta=0.01$}
		\addplot+[green, mark options={fill=green, scale=0.5, mark=square*, solid}] table  [x=avg queries, y=success rate, col sep=comma]{data/ablation_bandit_exploration_1.0_fd_eta_0.1.csv};
		\addlegendentry{$\delta=1$, $\eta=0.1$}
		\addplot+[teal, mark options={fill=teal, scale=0.5, mark=square*, solid}] table  [x=avg queries, y=success rate, col sep=comma]{data/ablation_bandit_exploration_1.0_fd_eta_1.0.csv};
		\addlegendentry{$\delta=1$, $\eta=1$}
		\addplot+[cyan, mark options={fill=cyan, scale=0.5, mark=square*, solid}] table  [x=avg queries, y=success rate, col sep=comma]{data/ablation_bandit_exploration_10.0_fd_eta_0.01.csv};
		\addlegendentry{$\delta=10$, $\eta=0.01$}
		\addplot+[blue, mark options={fill=blue, scale=0.5, mark=square*, solid}] table  [x=avg queries, y=success rate, col sep=comma]{data/ablation_bandit_exploration_10.0_fd_eta_0.1.csv};
		\addlegendentry{$\delta=10$, $\eta=0.1$}
		\addplot+[violet, mark options={fill=violet, scale=0.5, mark=square*, solid}] table  [x=avg queries, y=success rate, col sep=comma]{data/ablation_bandit_exploration_10.0_fd_eta_1.0.csv};
		\addlegendentry{$\delta=10$, $\eta=1$}
		\end{axis}
		\node[below] at ({5.6, 3.91}) {\pgfplotslegendfromname{attack}};
		\end{tikzpicture}
		
		\vspace{-0.2em}
		\caption{Success rate against the average number of queries with different hyperparameters. The square markers indicate the results of Bandits method. The numbers at round markers show the values of initial block size.}
		\label{fig:hyperparameter_sensitivity}
	\end{minipage}
\end{figure*}

\subsection{Untargeted attacks on ImageNet with smaller $\epsilon$}
To evaluate the performance of our method in a more constrained perturbation limit, we conducted experiments on ImageNet with the maximum perturbation $\epsilon \in \{0.01, 0.03\}$. The experiments are done in the untargeted attack setting. We restrict the maximum number of queries to be 10,000.

For NES and Bandits, which are gradient estimation based algorithms, the hyperparameters could be sensitive to the change of $\epsilon$. For this reason, we re-tuned for the hyperparameters for these methods. For NES, we tuned for samples per step $n \in \{50, 100, 200\}$, finite difference probe $\eta \in \{0.01, 0.1, 1\}$, and learning rate $h \in \{0.001, 0.005, 0.01\}$. For Bandits, we tuned for OCO learning rate $\eta \in \{1, 10, 100\}$, image learning rate $h \in \{0.001, 0.005, 0.01\}$, bandit exploration $\delta \in \{0.01, 0.1, 1\}$, and finite difference probe $\eta \in \{0.01, 0.1, 1\}$. The hyperparameters used are listed in supplementary B.4.

The results are shown in \cref{tab:imagenet_untargeted_epsilon} and \Cref{fig:cumulative_distribution_epsilon}. We can see that, for all $\epsilon$, our success rate is higher than the baseline methods. The results show that the margin in the success rate with respect to Bandits gets larger as $\epsilon$ decreases, up to 10\%, while maintaining the query efficiency lead.

\begin{table}[h]
	\centering
	\small
	\begin{adjustbox}{max width=\columnwidth}
		\begin{tabular}{>{\centering}m{1.2cm}>{\centering}m{1.1cm} >{\centering}m{1cm} >{\centering}m{1cm} >{\centering}m{1cm} | >{\centering}m{2cm}}
			\toprule[1pt]
			& \textbf{Method} & \textbf{Success rate} & \textbf{Avg. queries} & \textbf{Med. queries} & \textbf{Avg. queries (NES success)}\\
			\midrule[1pt]
			\multirow{4}{*}{$\epsilon=0.01$}
			& PGD & 99.5 \% & 20 & - & - \\
			\cmidrule{2-6}
			& NES & 48.2\% & 3598 & 3000 & 3598 \\
			& Bandits & 72.4\% & 2318 & 1374 & 1052 \\
			& Ours & \textbf{81.3\%} & \textbf{2141} & \textbf{1249} & \textbf{852} \\ 
			\midrule
			\multirow{4}{*}{$\epsilon=0.03$}
			& PGD & 99.9 \% & 20 & - & - \\
			\cmidrule{2-6}
			& NES & 74.8\% & 2049 & 1200 & 2049 \\
			& Bandits & 91.3\% & 1382 & 520 & 774 \\
			& Ours & \textbf{95.9\%} & \textbf{1129} & \textbf{420} & \textbf{537} \\ 
			\midrule
			\multirow{4}{*}{$\epsilon=0.05$}
			& PGD & 99.9 \% & 20 & - & - \\
			\cmidrule{2-6}
			& NES & 80.3\% & 1660 & 900 & 1660 \\
			& Bandits & 94.9\% & 1030 & 286 & 603 \\
			& Ours & \textbf{98.5\%} & \textbf{722} & \textbf{237} & \textbf{376} \\ 
			\midrule[1pt]
		\end{tabular}
	\end{adjustbox}
	
	\vspace{-0.2em}
	\caption{Results for $\ell_\infty$ untargeted attacks on ImageNet with $\epsilon \in \{0.01, 0.03, 0.05\}$.}
	\label{tab:imagenet_untargeted_epsilon}
\end{table}

\vspace{-0.5em}

\subsection{Hyperparameter sensitivity}

We measure the robustness of our method to hyperparameters compared to Bandits. Each method's robustness was measured by sweeping through hyperparameters and plotting their success rate and average queries, on ImageNet untargeted attack setting. For our method, we swept the initial block size $k \in \{32, 64, 128\}$. Note that $k$ is the only hyperparameter for our method. For Bandits, we swept through bandit exploration $\delta \in \{0.1, 1, 10\}$, finite difference probe $\eta \in \{0.01, 0.1, 1\}$, and tile size $\in \{25, 50, 100\}$. Also note that Bandit still has two other hyperparameters, which were left as the original setting for the experiment.

The results are shown in \Cref{fig:hyperparameter_sensitivity}. The figure shows that the proposed method maintains a high success rate with low variance as the hyperparameter $k$ changes. On the contrary, Bandits method shows relatively higher variance in both success rate and average queries. In our opinion, gradient estimation based methods, in general, are sensitive to the first order update hyperparameter settings as the ascent direction is approximated under limited query budget.

\section{Conclusion}\label{sec:conclusion}
Motivated by the observation that finding an adversarial perturbation can be viewed as computing solutions of linear programs under bounded feasible set, we have developed a discrete surrogate problem for practical black-box adversarial attacks. In contrast to the current state of the art methods, our method does not require estimating the gradient vector and thus becomes free of the update hyperparameters. Our experiments show the state of the art attack success rate at significantly lower average/median/NES\_success queries on both untargeted and targeted attacks on neural networks.

\section*{Acknowledgements}
This work was partially supported by Samsung Advanced Institute of Technology and Institute for Information \& Communications Technology Planning \& Evaluation (IITP) grant funded by the Korea government (MSIT) (No.2019-0-01367, BabyMind). Hyun Oh Song is the corresponding author.

\bibliography{main}
\bibliographystyle{icml2019}

\clearpage

\input{supp.tex}

\end{document}

%% file: supp.tex
\appendix
\section{Proofs}\label{sec:proofs}
\addtocounter{lemma}{0}
\addtocounter{definition}{3}

\subsection{Proof of Lemma 1}

\begin{lemma}
Let $\Scal$ be the solution obtained by performing the local search algorithm. Then $\Scal$ is a local optima.
\end{lemma}

\begin{proof}
Suppose $\Scal$ is not a local optima, then there exists an element $x$ that satisfies one of the followings: $x \in \Scal \; \text{and} \; F(\Scal \setminus \{x\}) \ge F(\Scal)$ or $x \in \Vcal \setminus \Scal \; \text{and} \; F(\Scal \cup \{x\}) \ge F(\Scal)$.
This means the algorithm must not terminate with $\Scal$. Contradiction.
\end{proof}

\subsection{Proof of Theorem 1}

Before proving Theorem 1, we introduce submodularity index (SmI) which is a measure of the degree of submodularity \cite{zhou16}.

\begin{definition}
The submodularity index \cite{zhou16} for a set function $F: 2^\Vcal \rightarrow \reals$, a set $\Lcal$, and a cardinality $k$ is defined as
\begin{align*}
    \lambda_F(\Lcal, k) = \min_{\substack{\Acal \subseteq \Lcal \\ \Scal \cap \Acal = \emptyset \\ |\Scal| \le k}} \left\{\phi_F(\Scal, \Acal) \triangleq \sum_{x \in \Scal} F_x(\Acal) - F_{\Scal}(\Acal)\right\},
\end{align*}
\end{definition}

where $F_\Scal(\Acal) = F(\Acal \cup \Scal) - F(\Acal)$. 

It is easy to verify $\forall \Ical \subseteq \Jcal$, SmI satisfies $\lambda_F(\Ical, k) \ge \lambda_F(\Jcal, k)$ and for the optimal solution $\Ccal$, $-2F(\Ccal) \le \lambda_F(\Vcal, 2) \le 2F(\Ccal)$. Following lemma bounds the degradation in submodularity with SmI.

\begin{lemma}
Let $\Acal$ be an arbitrary set, $\Bcal = \Acal \cup \{y_1, ..., y_M\}$ and $x \in \overline{\Bcal}$. Then, $F_x(\Acal) - F_x(\Bcal) \ge M \lambda_F(\Bcal, 2)$
\label{lem:smi}
\end{lemma}
\begin{proof} See \citet{zhou16} Lemma 3.\end{proof}

\begin{lemma}\label{lem:approx_submod}
Let $\Ycal$ be an arbitrary set and $\Acal \subseteq \Bcal$, Then 
\begin{align*}
F(\Acal \cup & \Ycal) - F(\Acal) \\
&\ge F(\Bcal \cup \Ycal) - F(\Bcal) + |\Bcal \setminus \Acal| \cdot |\Ycal| \cdot \lambda_F(\Bcal \cup \Ycal, 2)
\end{align*}
\end{lemma}

\begin{proof}
Let $\Ycal = \{a_1, ..., a_n\}$. Then,
\begin{align*}
&F(\Acal \cup \{a_1\}) - F(\Acal) \\
&\qquad\quad \ge F(\Bcal \cup \{a_1\}) - F(\Bcal) + |\Bcal \setminus \Acal| \lambda_F(\Bcal, 2) \\
&F(\Acal \cup \{a_1, a_2\}) - F(A \cup \{a_1\})\\
&\qquad\quad \ge F(\Bcal \cup \{a_1, a_2\}) - F(\Bcal \cup \{a_1\}) \\
&\qquad\quad\quad+ |\Bcal \setminus \Acal| \lambda_F(\Bcal \cup \{a_1\}, 2) \\
&\hspace{8em} \vdots\\
&F(\Acal \cup \Ycal) - F(\Acal \cup \Ycal \setminus \{a_n\})\\
&\qquad\quad \ge F(\Bcal \cup \Ycal) - F(\Bcal \cup \Ycal \setminus \{a_n\}) \\
&\qquad\quad\quad + |\Bcal \setminus \Acal| \lambda_F(\Bcal \cup \Ycal \setminus \{a_n\}, 2)
\end{align*}

By telescoping sum, 
\small
\begin{align*}
&F(\Acal \cup \Ycal) - F(\Acal)\\
&\ge F(\Bcal \cup \Ycal) - F(\Bcal) + |\Bcal \setminus \Acal| \sum_{i=1}^n \lambda_F(\Bcal \cup \{a_1, ... a_{i-1}\}, 2) \\
&\ge F(\Bcal \cup \Ycal) - F(\Bcal) + |\Bcal \setminus \Acal| \cdot |\Ycal| \cdot \lambda_F(\Bcal \cup \Ycal, 2) \tag{By property of SmI}
\end{align*}
\end{proof}
\normalsize

Next lemma relates the local optima solution from local search with submodularity index.

\begin{lemma}\label{lem:localopt_bound}
If $\Scal$ is a local optima for a function F, then for any subsets $\Ical \subseteq \Scal \subseteq \Jcal$, the following holds.

\small
\begin{align*}
&F(\Ical) \le F(\Scal) - \binom{|\Scal \setminus \Ical|}{2} \lambda_F(\Scal, 2) \\
&F(\Jcal) \le F(\Scal) - \binom{|\Jcal \setminus \Scal|}{2} \lambda_F(\Jcal, 2)
\end{align*}
\normalsize
\end{lemma}

\begin{proof} 
Let $\Ical = \Tcal_0 \subseteq \Tcal_1 \subseteq \dots \subseteq \Tcal_k = \Scal $ be a chain of sets where $\Tcal_i \setminus \Tcal_{i-1} = \{a_i\}$. For each $1 \le i \le k$, the following holds.

\small
\begin{align*}
    F(\Tcal_i)-F(\Tcal_{i-1}) & \ge F(\Scal)-F(\Scal \setminus \{a_i\}) + (k-i) \lambda_F(\Scal \setminus \{a_i\}, 2) \tag{By Lemma 3} \\ 
    & \ge (k-i) \lambda_F(\Scal \setminus \{a_i\}, 2) \tag{By the definition of local optima} \\ 
    & \ge (k-i) \lambda_F(\Scal, 2) \tag{By the property of SmI}  
\end{align*}
By telescoping sum,
\begin{align*}
    F(\Scal) - F(\Ical) & \ge \sum_{i=1}^k (k-i) \lambda_F(\Scal, 2) = \binom{|\Scal \setminus \Ical|}{2} \lambda_F(\Scal, 2)
\end{align*}
\normalsize

Similarly, Let $\Scal = \Tcal_0 \subseteq \Tcal_1 \subseteq \dots \subseteq \Tcal_k = \Jcal $ be a chain of sets where $\Tcal_i \setminus \Tcal_{i-1} = \{a_i\}$. For each $1 \le i \le k$, the following holds.

\small
\begin{align*}
    F(\Tcal_i)-F(\Tcal_{i-1}) & \le F(\Scal \cup \{a_i\}) - F(\Scal) - (i-1) \lambda_F(\Tcal_{i-1} , 2)  \tag{By Lemma 3} \\
    & \le -(i-1) \lambda_F(\Tcal_{i-1}, 2) \tag{By the definition of local optima} \\
    & \le -(i-1) \lambda_F(\Jcal, 2) \tag{By the property of SmI} 
\end{align*}
\normalsize
By telescoping sum,
\small
\begin{align*}
    F(\Jcal) - F(\Scal) & \le  -\sum_{i=1}^k (i-1) \lambda_F(\Jcal, 2) \\
    & = - \binom{|\Jcal \setminus \Scal|}{2} \lambda_F(\Jcal, 2)
\end{align*}
\normalsize
\end{proof}

Now, we prove Theorem 1.

\begin{theorem}
Let $\Ccal$ be an optimal solution for a function $F$ and $\Scal$ be the solution obtained by the local search algorithm. Then,
\begin{align*}
    2F(\Scal)+F(\Vcal \setminus \Scal) \ge  F(\Ccal) + \xi \lambda_F(\Vcal, 2),
\end{align*}
where 
\small
\begin{align*}
    \xi = \binom{|\Scal \setminus \Ccal|}{2} + \binom{|\Ccal \setminus \Scal|}{2} + |\overline{\Scal \cup \Ccal}| \cdot |\Scal| + |\Ccal \setminus \Scal| \cdot |\Scal \cap \Ccal|
\end{align*}
\label{thm:main}
\end{theorem}
\normalsize

\begin{proof} 
Since $\Scal$ is a local optimum, The following holds by \Cref{lem:localopt_bound}.
\begin{align*}
    F(\Scal) & \ge F(\Scal \cap \Ccal) + \binom{|\Scal \setminus \Ccal|}{2} \lambda_F(\Scal, 2) \\
    F(\Scal) & \ge F(\Scal \cup \Ccal) + \binom{|\Ccal \setminus \Scal|}{2} \lambda_F(\Scal \cup \Ccal, 2) \\
\end{align*}
Also from \Cref{lem:approx_submod}, we have,
\begin{align*}
&F(\Scal \cup \Ccal) +  F(\Vcal \setminus \Scal) \\
&\qquad \ge F(\Ccal \setminus \Scal) + F(\Vcal) + |\overline{\Scal \cup \Ccal}| \cdot |\Scal| \cdot \lambda_F(\Vcal, 2)\\
&\qquad \ge F(\Ccal \setminus \Scal) + |\overline{\Scal \cup \Ccal}| \cdot |\Scal| \cdot \lambda_F(\Vcal, 2) \tag{By non-negativity} \\
\end{align*}
Also,
\begin{align*}
&F(\Scal \cap \Ccal) +  F(\Ccal \setminus \Scal) \\
&\qquad \ge F(\Ccal) + F(\emptyset) + |\Ccal \setminus \Scal| \cdot |\Scal \cap \Ccal| \cdot \lambda_F(\Ccal, 2)\\
&\qquad \ge F(\Ccal) + |\Ccal \setminus \Scal| \cdot |\Scal \cap \Ccal| \cdot \lambda_F(\Ccal, 2)  \tag{By non-negativity} \\
\end{align*}

Summing the inequalities, we get
\begin{align*}
    2F(\Scal)+F(\Vcal \setminus \Scal) \ge  F(\Ccal) + & \binom{|\Scal \setminus \Ccal|}{2} \lambda_F(\Scal, 2) \\
    + & \binom{|\Ccal \setminus \Scal|}{2} \lambda_F(\Scal \cup \Ccal, 2) \\
    + & |\overline{\Scal \cup \Ccal}| \cdot |\Scal| \cdot \lambda_F(\Vcal, 2) \\
    + & |\Ccal \setminus \Scal| \cdot |\Scal \cap \Ccal| \cdot \lambda_F(\Ccal, 2) \\
\end{align*}
Since all $\lambda_F(\cdot,2)$'s are greater than or equal to $\lambda_F(\Vcal, 2)$ by the property of SmI, we get
\begin{align*}
2F(\Scal) + F(\Vcal \setminus \Scal) &\ge  F(\Ccal) + \Bigg[ \binom{|\Scal \setminus \Ccal|}{2} + \binom{|\Ccal \setminus \Scal|}{2} \\
&\quad+ |\overline{\Scal \cup \Ccal}| \cdot |\Scal| + |\Ccal \setminus \Scal| \cdot |\Scal \cap \Ccal|\Bigg] \lambda_F(\Vcal, 2) \\
& = F(\Ccal) + \xi \lambda_F(\Vcal, 2)
\end{align*}
\end{proof}

\section{Hyperparameters}

\subsection{Experiments on Cifar-10}

Hyperparameters for NES and Bandits on Cifar-10 dataset in untargeted setting are shown in \cref{tab:cifar_nes} and \cref{tab:cifar_bandit} respectively. Note that the hyperparameters are tuned in a setting where images are normalized in a scale of $[0, 1]$ to maintain consistency with the experiments on ImageNet dataset.

\begin{table}[ht]
\vspace{1em}
\centering
\small
\begin{adjustbox}{max width=\columnwidth}
\begin{tabular}{lc}
\toprule[1pt]
\textbf{Hyperparameter} & \textbf{Value}  \\
\midrule
$\sigma$ for NES & 0.001 \\
$n$, size of each NES population & 100 \\
$\eta$, learning rate & 0.01 \\
$\beta$, momentum & 0.9 \\
\bottomrule[1pt]
\end{tabular}
\end{adjustbox}
\vspace{-0.2em}
\caption{Hyperparameters for NES untargeted attack on Cifar-10.}
\label{tab:cifar_nes}
\end{table}

\begin{table}[ht]
\vspace{1em}
\centering
\small
\begin{adjustbox}{max width=\columnwidth}
\begin{tabular}{lc}
\toprule[1pt]
\textbf{Hyperparameter} & \textbf{Value}  \\
\midrule
$\eta$, OCO learning rate & 0.1 \\
$h$, image learning rate & 0.01 \\
$\delta$, bandit exploration & 0.1 \\
$\eta$, finite difference probe & 0.1 \\
tile size & 16 \\
\bottomrule[1pt]
\end{tabular}
\end{adjustbox}
\vspace{-0.2em}
\caption{Hyperparameters for Bandits untargeted attack on Cifar-10.}
\label{tab:cifar_bandit}
\end{table}

\newpage

\subsection{Untargeted attacks on ImageNet}

Hyperparameters for NES and Bandits on ImageNet dataset in untargeted setting are listed in \cref{tab:imagenet_untargeted_nes} and \cref{tab:imagenet_untargeted_bandit}. We use NES implementation from \citet{bandit}, since \citet{nes} conducted experiments only in the targeted setting.

\begin{table}[ht]
	\vspace{1em}
	\centering
	\small
	\begin{adjustbox}{max width=\columnwidth}
		\begin{tabular}{lc}
			\toprule[1pt]
			\textbf{Hyperparameter} & \textbf{Value}  \\
			\midrule
			$n$, sample per step & 100 \\
			$\eta$, finite difference probe & 0.1 \\
			$h$, image learning rate & 0.01 \\
			\bottomrule[1pt]
		\end{tabular}
	\end{adjustbox}
	\vspace{-0.2em}
	\caption{Hyperparameters for NES untargeted attack on ImageNet.}
	\label{tab:imagenet_untargeted_nes}
\end{table}

\begin{table}[H]
	\vspace{1em}
	\centering
	\small
	\begin{adjustbox}{max width=\columnwidth}
		\begin{tabular}{lc}
			\toprule[1pt]
			\textbf{Hyperparameter} & \textbf{Value}  \\
			\midrule
			$\eta$, OCO learning rate & 100 \\
			$h$, image learning rate & 0.01 \\
			$\delta$, bandit exploration & 1.0 \\
			$\eta$, finite difference probe & 0.1 \\
			tile size & 50 \\
			\bottomrule[1pt]
		\end{tabular}
	\end{adjustbox}
	\vspace{-0.2em}
	\caption{Hyperparameters for Bandits untargeted attack on ImageNet.}
	\label{tab:imagenet_untargeted_bandit}
\end{table}

\subsection{Targeted attacks on ImageNet}

Hyperparameters for NES targeted attack on ImageNet dataset are shown in \cref{tab:imagenet_targeted_nes}. All the hyperparameters except for momentum are referred from the original paper. For momentum, we tuned with range $\beta \in \{0.5, 0.7, 0.9\}$. The result of tuning momentum is in \cref{tab:tuning_momentum}. We choose $\beta=0.7$ which records the lowest average queries.

\vspace{-0.5em}

\begin{table}[H]
\vspace{1em}
\centering
\small
\begin{adjustbox}{max width=\columnwidth}
\begin{tabular}{lc}
\toprule[1pt]
\textbf{Hyperparameter} & \textbf{Value}  \\
\midrule
$\sigma$ for NES & 0.001 \\
$n$, size of each NES population & 50 \\
$\eta$, learning rate & 0.01 \\
$\beta$, momentum & 0.7 \\
\bottomrule[1pt]
\end{tabular}
\end{adjustbox}
\vspace{-0.2em}
\caption{Hyperparameters for NES targeted attack on ImageNet.}
\label{tab:imagenet_targeted_nes}
\end{table}

\vspace{-2em}

\begin{table}[H]
\vspace{1em}
\centering
\small
\begin{adjustbox}{max width=\columnwidth}
\begin{tabular}{>{\centering}m{1.6cm} >{\centering}m{1cm} >{\centering}m{1cm} >{\centering}m{1cm}}
\toprule[1pt]
\textbf{Momentum} & \textbf{Success rate} & \textbf{Avg. queries} & \textbf{Med. queries}\\
\midrule
0.5 & 99.2\% & 16977 & 13375 \\
0.7 & 99.7\% & \textbf{16284} & \textbf{12650} \\
0.9 & \textbf{99.8\%} & 16725 & 13525 \\
\bottomrule[1pt]
\end{tabular}
\end{adjustbox}
\vspace{-0.2em}
\caption{Result of tuning momentum for NES.}
\label{tab:tuning_momentum}
\end{table}

\subsection{Untargeted attacks on ImageNet with smaller $\epsilon$}
Hyperparameters for NES and Bandit with smaller maximum perturbation are given in \Cref{tab:imagenet_eps_nes} and \Cref{tab:imagenet_eps_bandit}. Since we run the experiments in untargeted setting, we use NES implementation from \citet{bandit}.

\begin{table}[H]
\vspace{1em}
\centering
\small
\begin{adjustbox}{max width=\columnwidth}
\begin{tabular}{lcc}
\toprule[1pt]
\multirow{2}{*}{\textbf{Hyperparameter}} & \multicolumn{2}{c}{\textbf{Value}} \\
\cmidrule{2-3}
& $\epsilon=0.01$ & $\epsilon=0.03$ \\
\midrule
$n$, samples per step & 100 & 100 \\
$\eta$, finite difference probe & 1 & 1 \\
$h$, image learning rate & 0.001 & 0.005 \\
\bottomrule[1pt]
\end{tabular}
\end{adjustbox}
\vspace{-0.2em}
\caption{Hyperparameters for NES untargeted attack on ImageNet with smaller $\epsilon$.}
\label{tab:imagenet_eps_nes}
\end{table}

\begin{table}[H]
\vspace{1em}
\centering
\small
\begin{adjustbox}{max width=\columnwidth}
\begin{tabular}{lcc}
\toprule[1pt]
\multirow{2}{*}{\textbf{Hyperparameter}} & \multicolumn{2}{c}{\textbf{Value}} \\
\cmidrule{2-3}
& $\epsilon=0.01$ & $\epsilon=0.03$ \\
\midrule
$\eta$, OCO learning rate & 100 & 100 \\
$h$, image learning rate & 0.001 &0.005 \\
$\delta$, bandit exploration & 0.1 & 1 \\
$\eta$, finite difference probe & 1 & 1 \\
tile size & 50 & 50 \\
\bottomrule[1pt]
\end{tabular}
\end{adjustbox}
\vspace{-0.2em}
\caption{Hyperparameters for Bandits untargeted attack on ImageNet with smaller $\epsilon$.}
\label{tab:imagenet_eps_bandit}
\end{table}

\section{Tuning Bandits for targeted attack}\label{sec:tune}

In applying Bandits to targeted attack, we tuned for image learning rate $h$ and OCO learning rate $\eta$. Other hyperparameters were set as the untargeted setting given by the authors. We performed grid search on the two hyperparameters, with range $h \in\{0.0001, 0.001, 0.005, 0.01, 0.05\}$ and $\eta\in\{1, 10, 100, 1000\}$. This sweep range covers the method's original untargeted setting, which is $h=0.01$ and $\eta=100$. Evaluation metrics were attack success rate and average queries. Results can be found below.

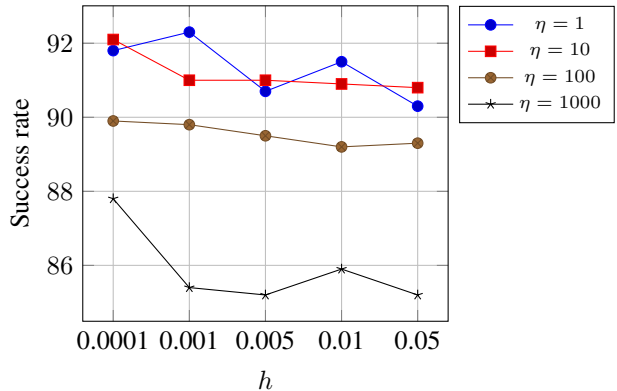
\begin{figure}[ht]
	\centering
	\begin{tikzpicture}
	\begin{axis}[
	width=0.78\columnwidth,
	height=0.7\columnwidth,
	legend pos={outer north east},
	xtick={0, 1, 2, 3, 4},
	xticklabels={0.0001, 0.001, 0.005, 0.01, 0.05},
	ylabel near ticks,
	xlabel={$h$},
	ylabel={Success rate},
	grid=major,
	scaled ticks = false,
	legend style={font=\scriptsize},
	]
	\addplot coordinates {
		(0, 91.8) (1, 92.3) (2, 90.7) (3, 91.5) (4, 90.3)
	};
	\addplot coordinates {
		(0, 92.1) (1, 91.0) (2, 91.0) (3, 90.9) (4, 90.8)
	};
	\addplot coordinates {
		(0, 89.9) (1, 89.8) (2, 89.5) (3, 89.2) (4, 89.3)
	};
	\addplot coordinates {
		(0, 87.8) (1, 85.4) (2, 85.2) (3, 85.9) (4, 85.2)
	};

	\legend{$\eta=1$, $\eta=10$,$\eta=100$,$\eta=1000$}
	\end{axis}
	\end{tikzpicture}
	\caption{Success rate with given hyperparameters.}
	\label{fig:bandit_tune_success} \vspace{0em}
\end{figure}

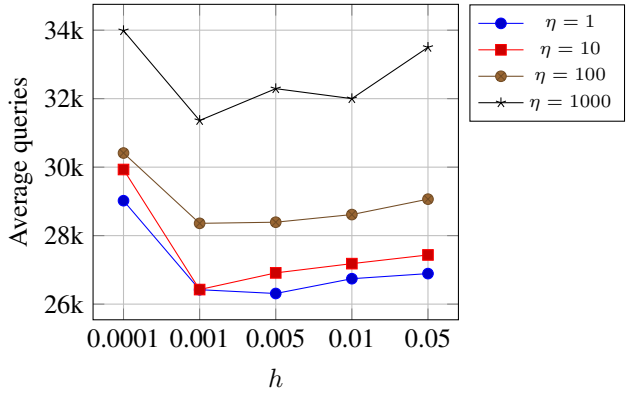
\begin{figure}[H]
	\centering
	\begin{tikzpicture}
	\begin{axis}[
	legend pos={outer north east},
	width=0.78\columnwidth,
	height=0.7\columnwidth,
	xtick={0, 1, 2, 3, 4},
	xticklabels={0.0001, 0.001, 0.005, 0.01, 0.05},
	ytick={26000, 28000, 30000, 32000, 34000},
	yticklabels={26k, 28k, 30k, 32k, 34k},
	ylabel near ticks,
	xlabel={$h$},
	ylabel={Average queries},
	ylabel style={at={(-0.14,0.5)}},
	grid=major,
	scaled ticks = false,
	legend style={font=\scriptsize},
	]
	\addplot coordinates {
		(0, 29020) (1, 26421) (2, 26310) (3, 26742) (4, 26891)
	};
	\addplot coordinates {
		(0, 29931) (1, 26425) (2, 26912) (3, 27182) (4, 27436)
	};
	\addplot coordinates {
		(0, 30412) (1, 28359) (2, 28392) (3, 28615) (4, 29065)
	};
	\addplot coordinates {
		(0, 33987) (1, 31363) (2, 32292) (3, 32001) (4, 33499)
	};
	
	\legend{$\eta=1$, $\eta=10$,$\eta=100$,$\eta=1000$}
	\end{axis}
	\end{tikzpicture}
	\caption{Average queries with given hyperparameters.}
	\label{fig:bandit_tune_queries} \vspace{0em}
\end{figure}

On the paper's Table 3 we used $h=0.001$ and $\eta=1$, which shows the best result on success rate with low average queries.

\section{Additional plot on hyperparameter sensitivity analysis}\label{sec:ablation}

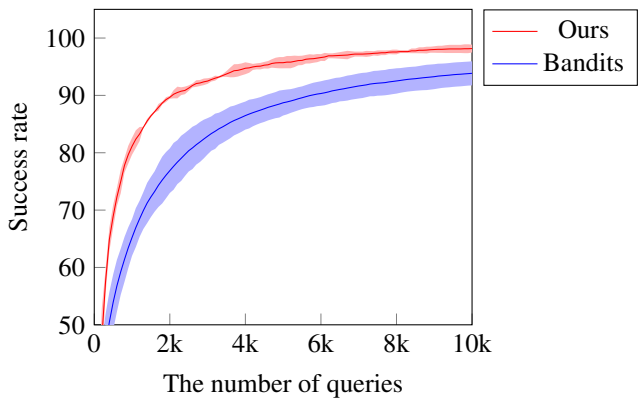
\begin{figure}[ht]
\centering
\begin{tikzpicture}
\begin{axis}[
width=0.8\columnwidth,
height=0.7\columnwidth,
xlabel={The number of queries},
ylabel={Success rate},
scaled ticks = false,
ylabel near ticks,
tick pos=left,
xtick={0, 2000, 4000, 6000, 8000, 10000},
xticklabels={0, 2k, 4k, 6k, 8k, 10k},
ytick={50, 60, 70, 80, 90, 100},
yticklabels={50, 60, 70, 80, 90, 100},
xmin=0,
xmax=10000,
ymin=50,
ymax=105,
no marks,
legend pos={outer north east},
]
\addplot[name path=min, draw=none, forget plot] table [x index={0}, y index={1}, col sep=comma]{./data/ablation_lls_mean_std.csv};
\addplot[name path=max, draw=none, forget plot] table [x index={0}, y index={2}, col sep=comma]{./data/ablation_lls_mean_std.csv};
\addplot[name path=avg, red] table [x index={0}, y index={3}, col sep=comma]{./data/ablation_lls_mean_std.csv};
\addplot[red!30, forget plot] fill between[of=min and max];
\addlegendentry{Ours}

\addplot[name path=min, draw=none, forget plot] table [x index={0}, y index={1}, col sep=comma]{./data/ablation_bandit_mean_std.csv};
\addplot[name path=max, draw=none, forget plot] table [x index={0}, y index={2}, col sep=comma]{./data/ablation_bandit_mean_std.csv};
\addplot[name path=avg, blue] table [x index={0}, y index={3}, col sep=comma]{./data/ablation_bandit_mean_std.csv};
\addplot[blue!30, forget plot] fill between[of=min and max];
\addlegendentry{Bandits}

\end{axis}
\end{tikzpicture}
\caption{Mean and standard deviation plots of success rate against the number of queries across different hyperparamters. The solid lines show the average success rate (y-axis) at each query limit (x-axis).}
\label{fig:cumulative_ablation} 
\end{figure}

To show the robustness of our method to hyperparameters more explicitly, we draw a mean and standard deviation plot of success rate against the number of queries across different hyperparameter settings for each attack method. The experimental protocol is the same as in Section 5.5 in the main text. The results are shown in \Cref{fig:cumulative_ablation}. The figure shows that our method is less sensitive to the hyperparameters than Bandits at every query limit.